\documentclass{article} 
\usepackage[final,nonatbib]{nips_2016}
\usepackage{times}
\usepackage{url}
\usepackage{subcaption}
\usepackage{graphicx} 
\usepackage{multirow}

\usepackage{amsmath,amssymb,amsthm,array,bm,xspace}
\usepackage{bibspacing}
\usepackage{stackengine}
\usepackage{scalerel}
\usepackage{algpseudocode}
\usepackage{algorithmicx}
\usepackage{algorithm}

\newcommand{\bbR}{\mathbb{R}}
\newcommand{\bbN}{\mathbb{N}}
\newcommand{\bb}{{\bm b}}
\newcommand{\bd}{{\bm d}}

\newcommand{\bh}{{\bm h}}
\newcommand{\bv}{{\bm v}}
\newcommand{\bx}{{\bm x}}
\newcommand{\by}{{\bm y}}
\newcommand{\bz}{{\bm z}}
\newcommand{\bone}{{\bm 1}}
\newcommand{\diag}{\mathrm{diag}}

\newcommand{\caP}{\mathcal{P}}
\newcommand{\caQ}{\mathcal{Q}}

\newcommand{\set}[1]{\{#1\}}
\newcommand{\rmd}{\mathrm{d}}

\newcommand{\bracket}[2]{\langle #1,#2\rangle}
\newcommand{\E}{\mathbf{E}}
\newcommand{\Var}{\mathop{\mathbf{Var}}}

\newcommand{\Lovasz}{Lov{\'a}sz\xspace}

\allowdisplaybreaks

\newtheorem{theorem}{Theorem}[section]
\newtheorem{lemma}[theorem]{Lemma}

\newtheorem{corollary}[theorem]{Corollary}

\title{Minimizing Quadratic Functions in Constant Time}
\author{Kohei Hayashi\\
  National Institute of Advanced Industrial Science and Technology\\
  \texttt{hayashi.kohei@gmail.com}
  \And
  Yuichi Yoshida\\
  National Institute of Informatics \emph{and} Preferred Infrastructure, Inc.\\
  \texttt{yyoshida@nii.ac.jp}}

\begin{document}
\maketitle

\begin{abstract}
  A sampling-based optimization method for quadratic functions is
  proposed. Our method approximately solves the following
  $n$-dimensional quadratic minimization problem in constant time,
  which is independent of $n$:
  $z^*=\min_{\bv \in \bbR^n}\bracket{\bv}{A \bv} +
  n\bracket{\bv}{\diag(\bd)\bv} + n\bracket{\bb}{\bv}$,
  where $A \in \bbR^{n \times n}$ is a matrix and $\bd,\bb \in \bbR^n$
  are vectors. Our theoretical analysis specifies the number of
  samples $k(\delta, \epsilon)$ such that the approximated solution
  $z$ satisfies $|z - z^*| = O(\epsilon n^2)$ with probability
  $1-\delta$. The empirical performance (accuracy and runtime) is
  positively confirmed by numerical experiments.
\end{abstract}


\section{Introduction}

A quadratic function is one of the most important function classes in
machine learning, statistics, and data mining. Many fundamental
problems such as linear regression, $k$-means clustering, principal
component analysis, support vector machines, and kernel
methods~\cite{Murphy:2012} can be formulated as a minimization problem
of a quadratic function.

In some applications, it is sufficient to compute the minimum value of
a quadratic function rather than its solution. For example,
Yamada~\emph{et~al.}~\cite{Yamada:2011} proposed an efficient method
for estimating the Pearson divergence, which provides useful
information about data, such as the density
ratio~\cite{Sugiyama:2012}. They formulated the estimation problem as
the minimization of a squared loss and showed that the Pearson
divergence can be estimated from the minimum value. The least-squares
mutual information~\cite{Suzuki:2011} is another example that can be
computed in a similar manner.


Despite its importance, the minimization of a quadratic function has
the issue of scalability. Let $n\in\bbN$ be the number of variables
(the ``dimension'' of the problem). In general, such a minimization
problem can be solved by quadratic programming (QP), which requires
$\mathrm{poly}(n)$ time. If the problem is convex and there are no
constraints, then the problem is reduced to solving a system of linear
equations, which requires $O(n^3)$ time. Both methods easily become
infeasible, even for medium-scale problems, say, $n>10000$.

Although several techniques have been proposed to accelerate quadratic
function minimization, they require at least linear time in $n$. This
is problematic when handling problems with an ultrahigh dimension, for
which even linear time is slow or prohibitive.
For example, stochastic gradient descent (SGD) is an optimization method
that is widely used for large-scale problems. A nice property of this method is that,
if the objective function is strongly convex, it outputs a point that is sufficiently close to an optimal solution after a constant
number of iterations~\cite{Bottou:2004}. Nevertheless, in each
iteration, we need at least $O(n)$ time to access the variables.
%
Another technique is low-rank approximation such as
Nystr\"{o}m's method~\cite{Williams:2001}.
The underlying idea is the approximation of the
problem by using a low-rank matrix, and by doing so, we can drastically reduce the time complexity. However, we still need to compute the
matrix--vector product of size $n$, which requires $O(n)$ time.
Clarkson~\emph{et~al.}~\cite{Clarkson:2012} proposed sublinear-time
algorithms for special cases of quadratic function
minimization. However, it is ``sublinear'' with respect to the number
of pairwise interactions of the variables, which is $O(n^2)$, and
their algorithms require $O(n \log^c n)$ time for some $c \geq 1$.

\paragraph{Our contributions:}

Let $A \in \bbR^{n \times n}$ be a matrix and $\bd,\bb \in \bbR^n$
be vectors.  Then, we consider the following quadratic problem:
\begin{align}
  \label{eq:the-problem}
  \operatornamewithlimits{minimize}_{\bv \in \bbR^n}~ & p_{n,A,\bd,\bb}(\bv),~\text{where}~p_{n,A,\bd,\bb}(\bv) = \bracket{\bv}{A \bv} + n\bracket{\bv}{\diag(\bd)\bv} + n\bracket{\bb}{\bv}.
\end{align}
Here, $\bracket{\cdot}{\cdot}$ denotes the inner product and $\diag(\bd)$ denotes the matrix whose diagonal entries are specified
by $\bd$.  Note that a constant term can be included
in~\eqref{eq:the-problem}; however, it is irrelevant when
optimizing~\eqref{eq:the-problem}, and hence we ignore it.

Let $z^* \in \bbR$ be the optimal value of~\eqref{eq:the-problem} and
let $\epsilon, \delta \in (0,1)$ be parameters.  Then, the main goal
of this paper is the computation of $z$ with
$|z - z^*| = O(\epsilon n^2)$ with probability at least $1-\delta$ in
\emph{constant time}, that is, independent of $n$.  Here, we assume
the real RAM model~\cite{Brattka:1998cv}, in which we can perform
basic algebraic operations on real numbers in one step.  Moreover, we
assume that we have query accesses to $A$, $\bb$, and $\bd$, with
which we can obtain an entry of them by specifying an index.  We note
that $\bz^*$ is typically $\Theta(n^2)$ because
$\langle \bv, A\bv\rangle$ consists of $\Theta(n^2)$ terms, and
$\langle \bv, \diag(\bd)\bv\rangle$ and $\langle \bb,\bv\rangle$
consist of $\Theta(n)$ terms.  Hence, we can regard the error of
$\Theta(\epsilon n^2)$ as an error of $\Theta(\epsilon)$ for each
term, which is reasonably small in typical situations.
%

Let $\cdot|_S$ be an operator that extracts a submatrix (or subvector)
specified by an index set $S\subset\bbN$; then, our algorithm is
defined as follows, where the parameter $k := k(\epsilon,\delta)$ will
be determined later.
\begin{algorithm}[H]
  \caption{}\label{alg:quadratic-form}
  \begin{algorithmic}[1]
    \Require{An integer $n \in \bbN$, query accesses to the matrix
      $A \in \bbR^{n \times n}$ and to the vectors
      $\bd,\bb \in \bbR^n$, and $\epsilon,\delta > 0$}
    \State{$S \leftarrow$ a sequence of $k = k(\epsilon,\delta)$
      indices independently and uniformly sampled from
      $\{1,2,\dots,n\}$}.  \State{\Return
      $\frac{n^2}{k^2}\min_{\bv \in
        \bbR^n}p_{k,A|_S,\bd|_S,\bb|_S}(\bv)$.}
  \end{algorithmic}
\end{algorithm}


In other words, we sample a constant number of indices from the set
$\set{1,2,\ldots,n}$, and then solve the
problem~\eqref{eq:the-problem} restricted to these indices.  Note that
the number of queries and the time complexity are $O(k^2)$ and
$\mathrm{poly}(k)$, respectively.  In order to analyze the difference
between the optimal values of $p_{n,A,\bd,\bb}$ and
$p_{k,A|_S,\bd|_S,\bb|_S}$, we want to measure the ``distances''
between $A$ and $A|_S$, $\bd$ and $\bd|_S$, and $\bb$ and $\bb|_S$,
and want to show them small.  To this end, we exploit graph limit
theory, initiated by \Lovasz and Szegedy~\cite{Lovasz:2006jj} (refer
to~\cite{Lovasz:2012wn} for a book), in which we measure the distance
between two graphs on different number of vertices by considering
continuous versions.  Although the primary interest of graph limit
theory is graphs, we can extend the argument to analyze matrices and
vectors.


Using synthetic and real settings, we demonstrate that our method is
orders of magnitude faster than standard polynomial-time algorithms
and that the accuracy of our method is sufficiently high.

\paragraph{Related work:}

Several constant-time approximation algorithms are known for combinatorial optimization problems such as the max cut problem on dense graphs~\cite{Frieze:1996er,Mathieu:2008vs}, constraint satisfaction problems~\cite{Alon:2002be,Yoshida:2011da}, and the vertex cover problem~\cite{Nguyen:2008fr,Onak:2012cl,Yoshida:2012jv}.
However, as far as we know, no such algorithm is known for continuous optimization problems.

A related notion is property
testing~\cite{Goldreich:1998wa,Rubinfeld:1996um}, which aims to design
constant-time algorithms that distinguish inputs satisfying some
predetermined property from inputs that are ``far'' from satisfying
it.  Characterizations of constant-time testable properties are known
for the properties of a dense graph~\cite{Alon:2009gn,Borgs:2006el}
and the affine-invariant properties of a function on a finite
field~\cite{Yoshida:2014tq,Yoshida:2016zz}.

\paragraph{Organization}
In Section~\ref{sec:pre}, we introduce the basic notions from graph
limit theory.  In Section~\ref{sec:cut-norm}, we show that we can
obtain a good approximation to (a continuous version of) a matrix by
sampling a constant-size submatrix in the sense that the optimizations
over the original matrix and the submatrix are essentially equivalent.
Using this fact, we prove the correctness of
Algorithm~\ref{alg:quadratic-form} in Section~\ref{sec:algorithm}.  We
show our experimental results in Section~\ref{sec:experiments}.




\section{Preliminaries}\label{sec:pre}

For an integer $n$, let $[n]$ denote the set $\set{1,2,\ldots,n}$.
The notation $a = b \pm c$ means that $b - c \leq a \leq b + c$.
In this paper, we only consider functions and sets that are measurable.


Let $S = (x_1,\ldots,x_k)$ be a sequence of $k$ indices in $[n]$.  For
a vector $\bv \in \bbR^n$, we denote the \emph{restriction} of $\bv$
to $S$ by $\bv|_S \in \bbR^k$; that is, $(\bv|_S)_i = v_{x_i}$ for
every $i \in [k]$.  For the matrix $A \in \bbR^{n \times n}$, we
denote the \emph{restriction} of $A$ to $S$ by
$A|_S \in \bbR^{k \times k}$; that is, $(A|_S)_{ij} = A_{x_ix_j}$ for
every $i,j \in [k]$.

\subsection{Dikernels}
Following~\cite{Lovasz:2013bv}, we call a (measurable) function $f:[0,1]^2 \to \bbR$ a \emph{dikernel}.
A dikernel is a generalization of a \emph{graphon}~\cite{Lovasz:2006jj}, which is symmetric and whose range is bounded in $[0,1]$.
We can regard a dikernel as a matrix whose index is specified by a real value in $[0,1]$.
We stress that the term dikernel has nothing to do with kernel methods.

For two functions $f,g:[0,1]\to \bbR$, we define their inner product as $\langle f,g \rangle = \int_0^1 f(x)g(x)\rmd x$.
For a dikernel $W:[0,1]^2 \to \bbR$ and a function $f:[0,1] \to \bbR$, we define a function $Wf:[0,1] \to \bbR$ as $(Wf)(x) = \bracket{W(x,\cdot)}{f}$.

Let $W:[0,1]^2 \to \bbR$ be a dikernel.  The \emph{$L_p$ norm}
$\|W\|_p$ for $p \geq 1$ and the \emph{cut norm} $\|W\|_\square$ of
$W$ are defined as $\|W\|_p = \Bigl(\int_{0}^1\int_0^1 |W(x,y)|^p \rmd x \rmd y\Bigr)^{1/p}$ and $\|W\|_\square = \sup_{S,T\subseteq [0,1]}\Bigl|\int_S \int_T W(x,y) \rmd x \rmd y \Bigr|$, respectively, where the supremum is over all pairs of subsets.
We note that these norms satisfy the triangle inequalities and $\|W\|_\square \leq \|W\|_1$.

Let $\lambda$ be a Lebesgue measure.  A map $\pi : [0, 1] \to [0, 1]$
is said to be \emph{measure-preserving}, if the pre-image
$\pi^{-1}(X)$ is measurable for every measurable set $X$, and
$\lambda(\pi^{-1}(X)) = \lambda(X)$.  A \emph{measure-preserving
  bijection} is a measure-preserving map whose inverse map exists and
is also measurable (and then also measure-preserving).  For a measure
preserving bijection $\pi:[0,1] \to [0,1]$ and a dikernel
$W:[0,1]^2 \to \bbR$, we define the dikernel $\pi(W):[0,1]^2 \to \bbR$
as $\pi(W)(x,y) = W(\pi(x),\pi(y))$.




\subsection{Matrices and Dikernels}
Let $W:[0,1]^2 \to \bbR$ be a dikernel and $S = (x_1,\ldots,x_{k})$ be a sequence of elements in $[0,1]$.
Then, we define the matrix $W|_S \in \bbR^{k \times k}$ so that $(W|_S)_{ij} = W(x_i,x_j)$.

We can construct the dikernel $\widehat{A}:[0,1]^2 \to \bbR$ from the
matrix $A \in \bbR^{n \times n}$ as follows.  Let
$I_1 = [0,\frac{1}{n}], I_2 = (\frac{1}{n},\frac{2}{n}], \ldots, I_n =
(\frac{n-1}{n},\ldots,1]$.
For $x \in [0,1]$, we define $i_n(x) \in [n]$ as a unique integer
such that $x \in I_i$.  Then, we define $\widehat{A}(x,y) = A_{i_n(x)i_n(y)}$.
The main motivation for creating a dikernel from a matrix is that, by
doing so, we can define the distance between two matrices $A$ and $B$
of different sizes via the cut norm, that is,
$\|\widehat{A} - \widehat{B}\|_\square$.

We note that the distribution of $A|_S$, where $S$ is a sequence of $k$ indices that are uniformly and independently sampled from $[n]$ exactly matches the distribution of $\widehat{A}|_S$, where $S$ is a sequence of $k$ elements that are uniformly and independently sampled from $[0,1]$.








\section{Sampling Theorem and the Properties of the Cut Norm}\label{sec:cut-norm}
In this section, we prove the following theorem, which states that,
given a sequence of dikernels $W^1,\ldots,W^T:[0,1]^2 \to [-L,L]$, we
can obtain a good approximation to them by sampling a sequence of
a small number of elements in $[0,1]$.  Formally, we prove the
following:
\begin{theorem}\label{the:sampling}
  Let $W^1,\ldots,W^T: [0,1]^2 \to [-L,L]$ be dikernels.
  Let $S$ be a sequence of $k$ elements uniformly and independently sampled from $[0,1]$.
  Then, with a probability of at least $1 - \exp(-\Omega(k T/ \log_2k))$,
  there exists a measure-preserving bijection $\pi:[0,1] \to [0,1]$ such that, for any functions $f,g:[0,1]\to [-K,K]$ and $t \in [T]$, we have
  \[
    |\bracket{f}{W^t g} - \bracket{f}{\pi(\widehat{W^t|_S}) g}| = O\Bigl(LK^2\sqrt{T/\log_2 k}\Bigr).
  \]
\end{theorem}

We start with the following lemma, which states that, if a dikernel $W:[0,1]^2 \to \bbR$ has a small cut norm, then $\bracket{f}{Wf}$ is negligible no matter what $f$ is.
Hence, we can focus on the cut norm when proving Theorem~\ref{the:sampling}.
\begin{lemma}\label{lem:small-cut-norm->negligible}
  Let $\epsilon \geq 0$ and $W:[0,1]^2\to \bbR$ be a dikernel with $\|W\|_\square \leq \epsilon$.
  Then, for any functions $f,g:[0,1] \to [-K,K]$, we have $|\bracket{f}{Wg}| \leq \epsilon K^2$.
\end{lemma}
\begin{proof}
  For $\tau \in \bbR$ and the function $h:[0,1]\to \bbR$, let $L_\tau(h) := \set{ x \in [0,1] \mid h(x) = \tau}$ be the level set of $h$ at $\tau$.
  For $f' = f/K$ and $g' = g/K$, we have
  \begin{align*}
    |\bracket{f}{Wg}|
    & = K^2|\bracket{f'}{Wg'}|
    = K^2\Bigl|\int_{-1}^1\int_{-1}^1 \tau_1 \tau_2 \int_{L_{\tau_1}(f')} \int_{L_{\tau_2}(g')} W(x,y)  \rmd x \rmd y  \rmd \tau_1 \rmd \tau_2\Bigr| \\
    & \leq K^2\int_{-1}^1\int_{-1}^1 |\tau_1| |\tau_2|  \left|\int_{L_{\tau_1}(f')} \int_{L_{\tau_2}(g')} W(x,y) \rmd x \rmd y \right|  \rmd \tau_1 \rmd \tau_2 \\
    & \leq \epsilon K^2 \int_{-1}^1\int_{-1}^1  |\tau_1| |\tau_2| \rmd \tau_1 \rmd \tau_2
    = \epsilon K^2.
    \qedhere
  \end{align*}
\end{proof}


To introduce the next technical tool, we need several definitions.  We
say that the partition $\caQ$ is a \emph{refinement} of the partition
$\caP= (V_1,\ldots,V_p)$ if $\caQ$ is obtained by splitting each set
$V_i$ into one or more parts.  The partition $\caP = (V_1,\ldots,V_p)$
of the interval $[0,1]$ is called an \emph{equipartition} if
$\lambda(V_i) = 1/p$ for every $i \in [p]$.  For the dikernel
$W:[0,1]^2 \to \bbR$ and the equipartition $\caP = (V_1,\ldots, V_p)$
of $[0, 1]$, we define $W_\caP:[0,1]^2 \to \bbR$ as the function
obtained by averaging each $V_i \times V_j$ for $i,j \in [p]$.
More formally, we define
\[
  W_\caP(x,y) = \frac{1}{\lambda(V_i)\lambda(V_j)}\int_{V_i \times V_j}W(x',y')\rmd x' \rmd y' = p^2\int_{V_i \times V_j}W(x',y')\rmd x' \rmd y',
\]
where $i$ and $j$ are unique indices such that $x \in V_i$ and $y \in V_j$, respectively.

The following lemma states that any function $W:[0,1]^2 \to \bbR$ can be well approximated by $W_\caP$ for the equipartition $\caP$ into a small number of parts.
\begin{lemma}[Weak regularity lemma for functions on ${[0,1]}^2$~\cite{Frieze:1996er}]\label{lem:regularity-lemma}
  Let $\caP$ be an equipartition of $[0,1]$ into $k$ sets.
  Then, for any dikernel $W:[0,1]^2 \to \bbR$ and $\epsilon > 0$,
  there exists a refinement $\caQ$ of $\caP$ with  $|\caQ| \leq k2^{C/\epsilon^2}$ for some constant $C > 0$ such that
  \[
    \|W -W_\caQ\|_\square  \leq \epsilon \|W\|_2.
  \]
\end{lemma}

\begin{corollary}\label{cor:regularity-lemma}
  Let $W^1,\ldots,W^T:[0,1]^2 \to \bbR$ be dikernels.
  Then, for any $\epsilon > 0$, there exists an equipartition $\caP$ into $|\caP| \leq 2^{CT/\epsilon^2}$ parts for some constant $C > 0$ such that, for every $t \in [T]$,
  \[
    \|W^t -W^t_\caP\|_\square  \leq \epsilon \|W^t\|_2.
  \]
\end{corollary}
\begin{proof}
  Let $\caP^0$ be a trivial partition, that is, a partition consisting
  of a single part $[n]$.  Then, for each $t \in [T]$, we iteratively
  apply Lemma~\ref{lem:regularity-lemma} with $\caP^{t-1}$, $W^t$, and
  $\epsilon $, and we obtain the partition $\caP^t$ into at most
  $|\caP^{t-1}|2^{C/\epsilon^2}$ parts such that
  $\|W^t -W^t_{\caP^t}\|_\square \leq \epsilon \|W^t\|_2$.  Since
  $\caP^{t}$ is a refinement of $\caP^{t-1}$, we have
  $\|W^i - W^i_{\caP^t}\|_\square \leq \|W^i -
  W^i_{\caP^{t-1}}\|_\square$
  for every $i \in [t-1]$.  Then, $\caP^T$ satisfies the desired
  property with
  $|\caP^T| \leq (2^{C/\epsilon^2})^T = 2^{CT/\epsilon^2}$.
\end{proof}

As long as $S$ is sufficiently large, $W$ and $\widehat{W|_S}$ are close in the cut norm:
\begin{lemma}[(4.15) of~\cite{Borgs:2008hd}]\label{lem:cut-norm-concentrates}
  Let $W:[0,1]^2  \to [-L,L]$ be a dikernel and $S$ be a sequence of $k$ elements uniformly and independently sampled from $[0,1]$.
  Then, we have
  \[
    -\frac{2L}{k}\leq \E_S \|\widehat{W|_S}\|_\square-\|W\|_\square < \frac{8L}{k^{1/4}}.
  \]
\end{lemma}

Finally, we need the following concentration inequality.
\begin{lemma}[Azuma's inequality]
  Let $(\Omega, A, P)$ be a probability space, $k$ be a positive
  integer, and $C > 0$.  Let $\bz = (z_1,\ldots , z_k)$, where
  $z_1,\ldots , z_k$ are independent random variables, and $z_i$ takes
  values in some measure space $(\Omega_i, A_i)$.  Let
  $f : \Omega_1 \times \cdots \times \Omega_k \to \bbR$ be a function.
  Suppose that $|f(\bx)-f(\by)| \leq C$ whenever $\bx$ and $\by$ only
  differ in one coordinate.  Then
  \[
    \Pr\Bigl[|f(\bz) - \E_\bz[f(\bz)]|  > \lambda C\Bigr]  < 2e^{-\lambda^2/2k}.
  \]
\end{lemma}

Now we prove the counterpart of Theorem~\ref{the:sampling} for the cut norm.
\begin{lemma}\label{lem:sampling-lemma}
  Let $W^1,\ldots,W^T: [0,1]^2 \to [-L,L]$ be dikernels.
  Let $S$ be a sequence of $k$ elements uniformly and independently sampled from $[0,1]$.
  Then, with a probability of at least $1 - \exp(-\Omega(k T/ \log_2k))$,
  there exists a measure-preserving bijection $\pi:[0,1] \to [0,1]$ such that, for every $t \in [T]$, we have
  \[
    \|W^t - \pi(\widehat{W^t|_S})\|_{\square} = O\Bigl(L\sqrt{T/\log_2 k}\Bigr).
  \]
\end{lemma}
\begin{proof}
  First, we bound the expectations and then prove their concentrations.
  We apply Corollary~\ref{cor:regularity-lemma} to $W^1,\ldots,W^T$ and $\epsilon$, and let $\caP = (V_1,\ldots,V_p)$ be the obtained partition with $p \leq 2^{CT/\epsilon^2}$ parts such that
  \[
    \|W^t - W^t_\caP\|_\square  \leq  \epsilon L.
  \]
  for every $t \in [T]$.
  By Lemma~\ref{lem:cut-norm-concentrates}, for every $t \in [T]$, we have
  \[
    \E_S \|\widehat{W^t_\caP|_S}- \widehat{W^t|_S}\|_\square  =
    \E_S
    \|
    \stackon[-8pt]{$(W^t_\caP - W^t)|_S$}{\vstretch{1.5}{\hstretch{2.4}{\widehat{\phantom{\;\;\;\;\;\;\;\;\;}}}}}
    \|_\square
    \leq  \epsilon L + \frac{8L}{k^{1/4}}.
  \]
  Then, for any measure-preserving bijection $\pi:[0,1] \to [0,1]$ and $t \in [T]$, we have
  \begin{align}
    \E_S\|W^t - \pi(\widehat{W^t|_S})\|_\square
    & \leq  \|W^t - W^t_{\caP}\|_\square + \E_S\|W^t_\caP -  \pi(\widehat{W^t_\caP|_S})\|_\square  + \E_S\|\pi(\widehat{W^t_\caP|_S}) - \pi(\widehat{W^t|_S})\|_\square \nonumber \\
    & \leq 2\epsilon L+\frac{8L}{k^{1/4}} +\E_S \|W^t_\caP - \pi(\widehat{W^t_\caP|_S})\|_\square. \label{eq:as-a-consequence}
  \end{align}

  Thus, we are left with the problem of sampling from $\caP$.  Let
  $S = \set{x_1,\ldots,x_k}$ be a sequence of independent random
  variables that are uniformly distributed in $[0, 1]$, and let $Z_i$
  be the number of points $x_j$ that fall into the set $V_i$.  It is
  easy to compute that
  \[
    \E[Z_i] = \frac{k}{p} \quad \text{and} \quad \Var[Z_i] = \Bigl(\frac{1}{p} - \frac{1}{p^2}\Bigr)k < \frac{k}{p}.
  \]
  The partition $\caP'$ of $[0, 1]$ is constructed into the sets
  $V'_1, \ldots , V'_p$ such that $\lambda(V'_i) = Z_i/k$ and
  $\lambda(V_i \cap V'_i) = \min(1/p, Z_i/k)$.  For each $t \in [T]$,
  we construct the dikernel $\overline{W}^t:[0,1]\to \bbR$ such that the
  value of $\overline{W}^t$ on $V'_i \times V'_j$ is the same as the
  value of $W^t_\caP$ on $V_i \times V_j$.  Then, $\overline{W}^t$
  agrees with $W^t_\caP$ on the set
  $Q=\bigcup_{i,j \in [p]}(V_i\cap V'_i)\times (V_j\cap V_j')$.  Then,
  there exists a bijection $\pi$ such that
  $\pi(\widehat{W^t_\caP|_S}) = \overline{W}^t$ for each $t \in [T]$.
  Then, for every $t \in [T]$, we have
  \begin{align*}
    & \|W^t_\caP - \pi(\widehat{W^t_\caP|_S})\|_\square
    =
    \|W^t_\caP -\overline{W}^t\|_\square  \leq \|W^t_\caP -\overline{W}^t\|_1
    \leq  2L(1-\lambda(Q)) \\
    & = 2L\Bigl(1-  \Bigl( \sum_{i \in [p]}\min\Bigl(\frac{1}{p},\frac{Z_i}{k}\Bigr)\Bigr)^2\Bigr)
    \leq  4L\Bigl(1 - \sum_{i \in [p]} \min \Bigl(\frac{1}{p},\frac{Z_i}{k}\Bigr)\Bigr) \\
    & = 2L \sum_{i \in [p]} \Bigl| \frac{1}{p} - \frac{Z_i}{k}\Bigr| \leq 2L\Bigl(p\sum_{i \in [p]} \Bigl(\frac{1}{p} - \frac{Z_i}{k}\Bigr)^2\Bigr)^{1/2},
  \end{align*}
  which we rewrite as
  \[
    \|W^t_\caP - \pi(\widehat{W^t_\caP|_S})\|_\square^2 \leq  4L^2p \sum_{i \in [p]} \Bigl(\frac{1}{p} - \frac{Z_i}{k}\Bigr)^2.
  \]
  The expectation of the right hand side is $(4L^2p/k^2)\sum_{i \in [p]} \Var(Z_i) < 4L^2p/k$.
  By the Cauchy-Schwartz inequality, $\E\|W^t_\caP - \pi(\widehat{W^t_\caP|_S})\|_\square \leq  2L\sqrt{p/k}$.

  Inserted this into~\eqref{eq:as-a-consequence}, we obtain
  \[
    \E\|W^t - \pi(\widehat{W^t|_S})\|_\square \leq 2\epsilon L+ \frac{8L}{k^{1/4}} + 2L\sqrt{\frac{p}{k}} \leq 2\epsilon L+ \frac{8L}{k^{1/4}} + \frac{2L}{k^{1/2}}2^{CT/\epsilon^2}.
  \]
  Choosing $\epsilon = \sqrt{CT/(\log_2 k^{1/4})} = \sqrt{4CT/(\log_2 k)}$, we obtain the upper bound
  \[
    \E\|W^t - \pi(\widehat{W^t|_S})\|_\square \leq 2L\sqrt{\frac{4CT}{\log_2 k}} + \frac{8L}{k^{1/4}} + \frac{2L}{k^{1/4}} = O\Bigl(L\sqrt{\frac{T}{\log_2 k}}\Bigr).
  \]
  Observing that $\|W^t - \pi(\widehat{W^t|_S})\|_\square$ changes by at most $O(L/k)$ if one element in $S$ changes, we apply Azuma's inequality with $\lambda = k\sqrt{T/\log_2 k}$ and the union bound to complete the proof.
\end{proof}
The proof of Theorem~\ref{the:sampling} is immediately follows from
Lemmas~\ref{lem:small-cut-norm->negligible}
and~\ref{lem:sampling-lemma}.


\section{Analysis of Algorithm~\ref{alg:quadratic-form}}\label{sec:algorithm}

In this section, we analyze Algorithm~\ref{alg:quadratic-form}.
Because we want to use dikernels for the analysis, we introduce a
continuous version of $p_{n,A,\bd,\bb}$
(recall~\eqref{eq:the-problem}).  The real-valued function
$P_{n,A,\bd,\bb}$ on the functions $f:[0,1]\to \bbR$ is defined as
\[
  P_{n,A,\bd,\bb}(f) = \bracket{f}{\widehat{A} f} + \bracket{f^2}{\widehat{\bd \bone^\top} 1} + \bracket{f}{\widehat{\bb \bone^\top}1},
\]
where $f^2:[0,1]\to \bbR$ is a function such that $f^2(x) = f(x)^2$
for every $x \in [0,1]$ and $1:[0,1] \to \bbR$ is the constant
function that has a value of $1$ everywhere.  The following lemma
states that the minimizations of $p_{n,A,\bd,\bb}$ and
$P_{n,A,\bd,\bb}$ are equivalent:
\begin{lemma}\label{lem:equivalence}
  Let $A \in \bbR^{n \times n}$ be a matrix and $\bd,\bb \in \bbR^{n \times n}$ be vectors.
  Then, we have
  \[
    \min_{\bv \in [-K,K]^n}p_{n,A,\bd,\bb}(\bv) =
    n^2 \cdot \inf_{f:[0,1] \to [-K,K]}P_{n,A,\bd,\bb}(f).
  \]
  for any $K > 0$.
\end{lemma}
\begin{proof}
  First, we show that $n^2 \cdot \inf_{f:[0,1] \to [-K,K]}P_{n,A,\bd,\bb}(f) \leq \min_{\bv \in [-K,K]^n}p_{n,A,\bd,\bb}(\bv)$.
  Given a vector $\bv \in [-K,K]^n$, we define $f: [0,1] \to [-K,K]$ as $f(x) = v_{i_n(x)}$.
  Then,
  \begin{align*}
    \bracket{f}{\widehat{A}f} & =
    \sum_{i,j \in [n]} \int_{I_i}\int_{I_j}A_{ij} f(x)f(y) \rmd x \rmd y
    =
    \frac{1}{n^2}\sum_{i,j \in [n]} A_{ij}v_i v_j
    =
    \frac{1}{n^2}\bracket{\bv}{A\bv},   \\
    \bracket{f^2}{\widehat{\bd \bone^\top} 1} & =
    \sum_{i,j \in [n]} \int_{I_i}\int_{I_j}d_i f(x)^2 \rmd x \rmd y
    =
    \sum_{i \in [n]} \int_{I_i}d_i f(x)^2 \rmd x
    =
    \frac{1}{n}\sum_{i \in [n]}d_iv_i^2
    =
    \frac{1}{n}\bracket{\bv}{\diag(\bd) \bv},   \\
    \bracket{f}{\widehat{\bb \bone^\top}1} & =
    \sum_{ i,j \in [n]} \int_{I_i}\int_{I_j}b_i f(x) \rmd x \rmd y
    =
    \sum_{ i \in [n]} \int_{I_i}b_i f(x) \rmd x
    =
    \frac{1}{n}\sum_{i \in [n]}b_i v_i
    =
    \frac{1}{n}\bracket{\bv}{\bb}.
  \end{align*}
  Then, we have $n^2 P_{n,A,\bd,\bb}(f) \leq p_{n,A,\bd,\bb}(\bv)$.

  Next, we show that $\min_{\bv \in [-K,K]^n}p_{n,A,\bd,\bb}(\bv) \leq n^2 \cdot \inf_{f:[0,1] \to [-K,K]}P_{n,A,\bd,\bb}(f)$.
  Let $f:[0,1] \to [-K,K]$ be a measurable function.
  Then, for $x \in [0,1]$, we have
  \[
    \frac{\partial P_{n,A,\bd,\bb}(f(x))}{\partial f(x)} = \sum_{i \in [n]}\int_{I_i}A_{i i_n(x)}f(y)\rmd y + \sum_{j \in [n]}\int_{I_j}A_{i_n(x)j}f(y)\rmd y + 2d_{i_n(x)}f(x) + b_{i_n(x)}.
  \]
  Note that the form of this partial derivative only depends on $i_n(x)$; hence, in the optimal solution $f^*:[0,1] \to [-K,K]$, we can assume $f^*(x) = f^*(y)$ if $i_n(x) = i_n(y)$.
  In other words, $f^*$ is constant on each of the intervals $I_1,\ldots,I_n$.
  For such $f^*$, we define the vector $\bv\in \bbR^n$ as $v_i = f^*(x)$, where $x \in [0,1]$ is any element in $I_i$.
  Then, we have
  \begin{align*}
    \bracket{\bv}{A\bv} & = \sum_{i,j\in [n]}A_{ij}v_iv_j
    =
    n^2\sum_{i,j\in [n]}\int_{I_i}\int_{I_j}A_{ij}f^*(x)f^*(y)\rmd x \rmd y
    =
    n^2\bracket{f^*}{\widehat{A}f^*}, \\
    \bracket{\bv}{\diag(\bd) \bv}
    & = \sum_{i\in [n]}d_iv_i^2
    = n\sum_{i\in [n]}\int_{I_i}d_i f^*(x)^2\rmd x
    = n\bracket{(f^*)^2}{\widehat{\bd \bone^T}1},\\
    \bracket{\bv}{\bb}
    & = \sum_{i\in [n]}b_iv_i
    = n\sum_{i\in [n]}\int_{I_i}b_i f^*(x)\rmd x
    = n\bracket{f^*}{\widehat{\bb \bone^T}1}.
  \end{align*}
  Finally, we have
  $p_{n,A,\bd,\bb}(\bv) \leq n^2 P_{n,A,\bd,\bb}(f^*)$.
\end{proof}

Now we show that Algorithm~\ref{alg:quadratic-form} well-approximates the optimal value of~\eqref{eq:the-problem} in the following sense:
\begin{theorem}\label{the:error}
  Let $\bv^*$ and $z^*$ be an optimal solution and the optimal value,
  respectively, of problem~\eqref{eq:the-problem}.  By choosing
  $k(\epsilon,\delta) =
  2^{\Theta(1/\epsilon^2)}+\Theta(\log\frac{1}{\delta}\log \log
  \frac{1}{\delta})$,
  with a probability of at least $1-\delta$, a sequence $S$ of $k$
  indices independently and uniformly sampled from $[n]$ satisfies the
  following: Let $\tilde{\bv}^*$ and $\tilde{z}^*$ be an optimal
  solution and the optimal value, respectively, of the problem
  $\min_{\bv \in \bbR^k}p_{k,A|_S,\bd|_S,\bb|_S}(\bv)$.  Then, we have
  \[
    \Bigl|\frac{n^2}{k^2}\tilde{z}^* - z^*\Bigr| \leq \epsilon LK^2n^2,
  \]
  where $K = \max\set{\max_{i \in [n]}|v^*_i|, \max_{i \in [n]}|\tilde{v}^*_i|}$ and $L = \max\set{\max_{i,j}|A_{ij}|,\max_i |d_i|, \max_i |b_i|}$.
\end{theorem}
\begin{proof}

  We instantiate Theorem~\ref{the:sampling} with
  $k = 2^{\Theta(1/\epsilon^2)}+\Theta(\log\frac{1}{\delta}\log \log
  \frac{1}{\delta})$
  and the dikernels $\widehat{A}$, $\widehat{\bd \bone^\top}$, and
  $\widehat{\bb \bone^\top}$.  Then, with a probability of at least
  $1-\delta$, there exists a measure preserving bijection
  $\pi:[0,1]\to [0,1]$ such that
  \[
    \max \Bigl\{
    |\bracket{f}{(\widehat{A} - \pi(\widehat{A|_S}))f}|,
    |\bracket{f^2}{(\widehat{\bd \bone^\top} - \pi(\widehat{\bd \bone^\top|_S}))1}| ,
    |\bracket{f}{(\widehat{\bb \bone^\top} - \pi(\widehat{\bb \bone^\top|_S}))1}|\Bigr\} \leq \frac{\epsilon LK^2}{3}
  \]
  for any function $f:[0,1]\to [-K,K]$.
  Then, we have
  \begin{align*}
    \tilde{z}^*
    & =
    \min_{\bv \in \bbR^k}p_{k,A|_S,\bd|_S,\bb|_S}(\bv)
    = \min_{\bv \in [-K,K]^k}p_{k,A|_S,\bd|_S,\bb|_S}(\bv)\\
    & =
    k^2 \cdot \inf_{f:[0,1]\to [-K,K]} P_{k,A|_S,\bd|_S,\bb|_S} (f) \tag{By Lemma~\ref{lem:equivalence}}  \\
    & =
    k^2 \cdot \inf_{f:[0,1]\to [-K,K]}\Bigl(\bracket{f}{(\pi(\widehat{A|_S}) - \widehat{A})f} + \bracket{f}{\widehat{A} f} + \bracket{f^2}{(\pi(\widehat{\bd \bone^\top|_S}) - \widehat{\bd \bone^\top})1} + \bracket{f^2}{\widehat{\bd \bone^\top}1} + \\
    & \qquad \qquad \qquad \qquad \bracket{f}{(\pi(\widehat{\bb \bone^\top|_S}) - \widehat{\bb \bone^\top})1} + \bracket{f}{\widehat{\bb \bone^\top}1}\Bigr)\\
    & \leq
    k^2 \cdot \inf_{f:[0,1]\to [-K,K]} \Bigl(\bracket{f}{\widehat{A} f}+\bracket{f^2}{\widehat{\bd \bone^\top}1 } +\bracket{f}{\widehat{\bb \bone^\top}1 } \pm  \epsilon LK^2 \Bigr)\\
    & =
    \frac{k^2}{n^2} \cdot \min_{\bv \in [-K,K]^n} p_{n,A,\bd,\bb}(\bv) \pm \epsilon LK^2 k^2. \tag{By Lemma~\ref{lem:equivalence}} \\
    & =
    \frac{k^2}{n^2} \cdot \min_{\bv \in \bbR^n} p_{n,A,\bd,\bb}(\bv) \pm \epsilon LK^2 k^2 = \frac{k^2}{n^2}z^* \pm \epsilon LK^2 k^2.
  \end{align*}
  Rearranging the inequality, we obtain the desired result.
\end{proof}
We can show that $K$ is bounded when $A$ is symmetric and full rank.
To see this, we first note that we can assume $A + n\diag(\bd)$ is positive-definite, as otherwise $p_{n,A,\bd,\bb}$ is not bounded and the problem is uninteresting.
Then, for any set $S \subseteq [n]$ of $k$ indices, $(A+n\diag(\bd))|_S$ is again positive-definite because it is a principal submatrix.
Hence, we have $\bv^* = (A+n\diag(\bd))^{-1}n\bb/2$ and $\tilde{\bv}^* = (A|_S+n\diag(\bd|_S))^{-1}n\bb|_S/2$, which means that $K$ is bounded.



\section{Experiments}\label{sec:experiments}

In this section, we demonstrate the effectiveness of our method by
experiment.  All experiments were conducted on an Amazon EC2
c3.8xlarge instance. Error bars indicate the standard deviations over
ten trials with different random seeds.

\begin{figure}[tb]
\centering
  \begin{minipage}[t]{.4\linewidth}
  \vspace{0pt}
  \begin{figure}[H]
  \centering
    \centering
    \includegraphics[width=1\linewidth]{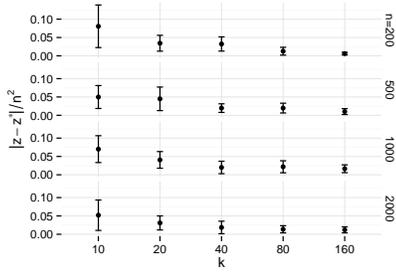}
  \caption{Numerical simulation: absolute approximation error scaled by $n^2$.}\label{fig:toy}
  \end{figure}
  \end{minipage}
  \begin{minipage}[t]{.5\linewidth}
  \vspace{0pt}
  \begin{table}[H]
    \small
    \caption{Pearson divergence: runtime (second).}\label{tab:kernel-time}
    \centering
    \begin{tabular}{rrrrrr}
  \hline
 & k & $n=500$ & 1000 & 2000 & 5000 \\ 
  \hline
\multirow{4}{*}{\rotatebox{90}{Proposed}}  & 20 & $0.002$ & $0.002$ & $0.002$ & $0.002$ \\ 
   & 40 & $0.003$ & $0.003$ & $0.003$ & $0.003$ \\ 
   & 80 & $0.007$ & $0.007$ & $0.008$ & $0.008$ \\ 
   & 160 & $0.030$ & $0.030$ & $0.033$ & $0.035$ \\ 
  \hline
\multirow{4}{*}{\rotatebox{90}{Nystr\"{o}m}} & 20 & $0.005$ & $0.012$ & $0.046$ & $0.274$ \\ 
   & 40 & $0.010$ & $0.022$ & $0.087$ & $0.513$ \\ 
   & 80 & $0.022$ & $0.049$ & $0.188$ & $0.942$ \\ 
   & 160 & $0.076$ & $0.116$ & $0.432$ & $1.972$ \\ 
   \hline
\end{tabular}

  \end{table}
  \end{minipage}
\end{figure}


\begin{table}[tb]
\small
  \caption{Pearson divergence: absolute approximation error.}
  \label{tab:kernel-error}
  \centering
  \begin{tabular}{rrrrrr}
  \hline
 & k & $n=500$ & 1000 & 2000 & 5000 \\ 
  \hline
\multirow{4}{*}{\rotatebox{90}{Proposed}}  & 20 & $0.0027\pm0.0028$ & $0.0012\pm0.0012$ & $0.0021\pm0.0019$ & $0.0016\pm0.0022$ \\ 
   & 40 & $0.0018\pm0.0023$ & $0.0006\pm0.0007$ & $0.0012\pm0.0011$ & $0.0011\pm0.0020$ \\ 
   & 80 & $0.0007\pm0.0008$ & $0.0004\pm0.0003$ & $0.0008\pm0.0008$ & $0.0007\pm0.0017$ \\ 
   & 160 & $0.0003\pm0.0003$ & $0.0002\pm0.0001$ & $0.0003\pm0.0003$ & $0.0002\pm0.0003$ \\ 
  \hline
\multirow{4}{*}{\rotatebox{90}{Nystr\"{o}m}} & 20 & $0.3685\pm0.9142$ & $1.3006\pm2.4504$ & $3.1119\pm6.1464$ & $0.6989\pm0.9644$ \\ 
   & 40 & $0.3549\pm0.6191$ & $0.4207\pm0.7018$ & $0.9838\pm1.5422$ & $0.3744\pm0.6655$ \\ 
   & 80 & $0.0184\pm0.0192$ & $0.0398\pm0.0472$ & $0.2056\pm0.2725$ & $0.5705\pm0.7918$ \\ 
   & 160 & $0.0143\pm0.0209$ & $0.0348\pm0.0541$ & $0.0585\pm0.1112$ & $0.0254\pm0.0285$ \\ 
   \hline
\end{tabular}

\end{table}


\paragraph{Numerical simulation}

We investigated the actual relationships between $n$, $k$, and
$\epsilon$. To this end, we prepared synthetic data as follows.
We randomly generated inputs as $A_{ij}\sim U_{[-1,1]}$,
$d_i\sim U_{[0,1]}$, and $b_i\sim U_{[-1,1]}$ for $i,j\in [n]$, where
$U_{[a,b]}$ denotes the uniform distribution with the support
$[a,b]$. After that, we solved~\eqref{eq:the-problem} by using
Algorithm~\ref{alg:quadratic-form} and compared it with the exact
solution obtained by QP.\footnote{We used GLPK
  (\url{https://www.gnu.org/software/glpk/}) for the QP solver.}
The result (Figure~\ref{fig:toy}) show the approximation errors were
evenly controlled regardless of $n$, which meets the error analysis
(Theorem~\ref{the:error}).

\paragraph{Application to kernel methods}

Next, we considered the kernel approximation of the Pearson
divergence~\cite{Yamada:2011}. The problem is defined as follows.
Suppose we have the two different data sets
$\bx=(x_1,\dots,x_{n})\in\bbR^n$ and
$\bx'=(x'_{1},\dots,x'_{n'})\in\bbR^{n'}$ where $n,n'\in\bbN$. Let
$H\in\bbR^{n\times n}$ be a gram matrix such that $H_{l,m} = \frac{\alpha}{n}\sum_{i=1}^n \phi(x_{i},x_{l})\phi(x_{i},x_{m}) +
\frac{1-\alpha}{n'}\sum_{j=1}^{n'}
\phi(x'_{j},x_{l})\phi(x'_{j},x_{m})$,
where $\phi(\cdot,\cdot)$ is a kernel function and $\alpha\in(0,1)$ is
a parameter. Also, let $\bh\in\bbR^{n}$ be a vector such that $h_l = \frac{1}{n}\sum_{i=1}^n \phi(x_{i},x_{l})$.
Then, an estimator of the $\alpha$-relative Pearson divergence between
the distributions of $\bx$ and $\bx'$ is obtained by
$-\frac{1}{2}-\min_{\bv \in \bbR^n}\frac{1}{2}\bracket{\bv}{H \bv} -
\bracket{\bh}{\bv} + \frac{\lambda}{2}\bracket{\bv}{\bv}.$
Here, $\lambda>0$ is a regularization parameter. In this experiment,
we used the Gaussian kernel $\phi(x,y)=\exp((x-y)^2/2\sigma^2)$ and
set $n'=200$ and $\alpha=0.5$; $\sigma^2$ and $\lambda$ were chosen by 5-fold
cross-validation as suggested in~\cite{Yamada:2011}. We randomly
generated the data sets as $x_i\sim N(1,0.5)$ for $i\in[n]$ and
$x'_j\sim N(1.5,0.5)$ for $j\in[n']$ where $N(\mu,\sigma^2)$ denotes
the Gaussian distribution with mean $\mu$ and variance $\sigma^2$.

We encoded this problem into \eqref{eq:the-problem} by setting
$A=\frac{1}{2}H$, $\bb=-\bh$, and $\bd=\frac{\lambda}{2n}\mathbf{1}_n$,
where $\mathbf{1}_n$ denotes the $n$-dimensional vector whose elements
are all one. After that, given $k$, we computed the second step of
Algorithm~\ref{alg:quadratic-form} with the pseudoinverse of
$A|_S+k\diag(\bd|_S)$.
Absolute approximation errors and runtimes were compared with
Nystr\"{o}m's method whose approximated rank was set to $k$. In terms
of accuracy, our method clearly outperformed Nystr\"{o}m's method
(Table~\ref{tab:kernel-error}). In addition, the runtimes of our
method were nearly constant, whereas the runtimes of Nystr\"{o}m's
method grew linearly in $k$ (Table~\ref{tab:kernel-time}).



\section{Acknowledgments}
We would like to thank Makoto Yamada for suggesting a motivating problem of our method.
K.~H. is supported by MEXT KAKENHI 15K16055.
Y.~Y.\ is supported by MEXT Grant-in-Aid for Scientific Research on Innovative Areas (No.~24106001), JST, CREST, Foundations of Innovative Algorithms for Big Data, and JST, ERATO, Kawarabayashi Large Graph Project.

\bibliographystyle{abbrv}
{
  \small
  \bibliography{main}
}

\end{document}